\pdfoutput=1 
\documentclass[letterpaper,oneside,11pt]{article}
\usepackage{pslatex}               
\usepackage[margin=1in]{geometry}   
\usepackage{bm}                  
\usepackage{url}                 
\usepackage{soul}                
\usepackage{float}               
\usepackage{natbib}
\usepackage{amssymb}             
\usepackage{optprog}             
\usepackage{graphicx}            
\usepackage{pdflscape}           
\usepackage{subfigure}           
\usepackage[tbtags]{amsmath}     

\usepackage{enumitem}            
   
   \newenvironment{compenum}{\begin{enumerate}[topsep=0pt,itemsep=0pt,parsep=0pt,partopsep=0pt]}{\end{enumerate}}

\usepackage{amsthm}              
   \theoremstyle{plain}
   \newtheorem{theorem}{Theorem}             
   \newtheorem{lemma}[theorem]{Lemma}
   
   \newtheorem{definition}[theorem]{Definition}

\usepackage[hang,small,bf,margin=20pt,position=top]{caption}
\renewcommand{\a}{\bm{a}}
\renewcommand{\b}{\bm{b}}
\renewcommand{\c}{\bm{c}}
\renewcommand{\o}{\omega}

\newcommand{\x}{\bm{x}}
\newcommand{\z}{\bm{z}}

\newcommand{\nn}{\phantom{0}}

\newcommand{\abs}[1]{\left| #1 \right|}
\newcommand{\sign}{\operatorname{sign}}
\renewcommand{\th}{\ensuremath{^{th}\ }}        
\newcommand{\set}[1]{\ensuremath{\left\{ #1 \right\}}}
\newcommand{\ceil}[1]{\left\lceil #1 \right\rceil}
\newcommand{\floor}[1]{\left\lfloor #1 \right\rfloor}
\newcommand{\Z}{\mathbb{Z}}

\newcommand{\rem}[1]{}

\newcommand{\sA}{\mathcal{A}}                
\newcommand{\sB}{\mathcal{B}}                
\newcommand{\sC}{\mathcal{C}}                
\newcommand{\intersection}{\bigcap}

\newcommand{\ap}{\a|_{(\sA\intersection \sB)}}
\newcommand{\bp}{\b|_{(\sA\intersection \sB)}}
\newcommand{\pte}{|_{(\sA\intersection \sB\intersection \sC)}} 

\newcommand\xc{\bm{x}^{(c)}}
\newcommand\xo{\bm{x}^{(o)}}
\newcommand\xdc{\bm{x}^{(cat)}}
\newcommand\xdo{\textsf{rank}\bigl(\xdc\bigr)}

\newcommand{\rank}[1]{\textsf{rank}\bigl(#1\bigr)}

\begin{document}
\title{Joint aggregation of cardinal and ordinal evaluations with an application to a student paper competition}
\author{Dorit S. Hochbaum \\
      Department of Industrial Engineering and Operations Research \\
      University of California, Berkeley \\
      email: {\tt hochbaum@ieor.berkeley.edu}\\
      Erick Moreno-Centeno \\
      Department of Industrial and Systems Engineering \\
      Texas A\&M University \\
      email: {\tt emc@tamu.edu}\\
}
\maketitle 

\begin{abstract}
An important problem in decision theory concerns the aggregation of individual
rankings/ratings into a collective evaluation. We illustrate a new aggregation
method in the context of the 2007 MSOM's student paper competition. The
aggregation problem in this competition poses two challenges. Firstly, each
paper was reviewed only by a very small fraction of the judges; thus the
aggregate evaluation is highly sensitive to the subjective scales chosen by the
judges. Secondly, the judges provided both cardinal and ordinal evaluations
(ratings and rankings) of the papers they reviewed. The contribution here is a
new robust methodology that jointly aggregates ordinal and cardinal evaluations
into a collective evaluation. This methodology is particularly suitable in
cases of incomplete evaluations---i.e., when the individuals evaluate only a
strict subset of the objects. This approach is potentially useful in managerial
decision making problems by a committee selecting projects from a large set or
capital budgeting involving multiple priorities.

\end{abstract}

\section{Introduction}\label{sec:introduction}
We present here a new framework for group decision making in which a group of
individuals, or judges, collectively ranks all of the objects in a universal
set.  This framework takes into consideration the pairwise comparisons implied
by the individuals' evaluations, and furthermore, it is first to combine
ordinal rankings with cardinal ratings so as to achieve an aggregate ranking
that represents as well as possible the individuals' assessments, as measured
by pre-set penalty functions.

Group-ranking problems are differentiated by whether the evaluations are given
in ordinal or cardinal scales. An ordinal evaluation, or \textit{ranking}, is
one where the objects are ordered from ``most preferred'' to ``least
preferred'' in the form of an ordered list (allowing ties).  On the other hand,
a cardinal evaluation, or \textit{rating}, is an assignment of scalars, which
are cardinal scores/grades, to the objects evaluated. In a rating, the
difference between the scores of two objects indicates the magnitude of
separation between such objects. Depending on the type of evaluations to be
aggregated, group-ranking problems are referred to  as \textit{ranking
aggregation problem} or \textit{rating aggregation problems}.

Previous work addressed either the rankings alone aggregation problem (e.g.
\citealt{KS62,Arr63,BTT89,HL06}), or the ratings alone aggregation problem
(e.g. \citealt{Kee76,Saa77}), but not both. One of the primary contributions
here is the technique that permits to jointly aggregate rankings and ratings
into a collective evaluation.
The individual evaluations input to a group-ranking problem can be
\textit{complete} or \textit{incomplete}. In cases when each individual in the
group ranks (rates) all of the objects in the universal set, then the ranking
(rating) is said to be complete, or {\em full list}; otherwise, it is said to
be incomplete, or {\em partial list}.
The framework developed here is applicable when the judges' ratings and
rankings are incomplete.

The power of the framework developed here is illustrated in ranking the
participants of the 2007 MSOM's student paper competition (\textbf{SPC}). This
SPC aggregate ranking problem poses challenges that are unique to that
scenario:
\begin{compenum}
\item The judges provided both ratings and rankings of the papers they
    reviewed.  This requires to reconcile the possibly conflicting two types of
    evaluations.
\item The incompleteness of the evaluations was extreme: Each judge evaluated
    fewer than a tenth of the papers, and each paper was reviewed by fewer than
    a tenth of the judges. This caused the aggregation to be subject to the
    ``incomplete evaluation'' phenomenon bias, in which the individual scales
    used by the judges affect the average scores, even if the preference
    ordering of all the judges agree with each other. Also, outlier scores that
    are too low or too high tend to dominate the aggregate score of the papers.
\end{compenum}

%

The issue of subjective scales is well recognized within the aggregate ranking
literature. \citet{Fre88}\rem{Ch. 8, page 296} argues that, the \emph{value
difference functions} (the rating scales) of two individuals involve an
arbitrary choice of scale and origin and thus the same numeric score from two
different judges generally do not have the same meaning. Similarly, in the
context of international surveys, a large number of studies (see, for example,
\citealt{BS01,Smi04,Har06}) show that the responses across different countries
do not have the same meaning. In particular, these studies showed that even
when asking respondents to rate each object using a simple 5-point rating
scale, there are significant differences in the \textit{response styles}
between countries. One example of a difference in response style, that arises
even when using a simple 5-point rating scale, is that in some countries there
is a tendency to use only the extreme categories while in others there is a
tendency to use only the middle categories. Another example of a difference in
response style, is that in some countries there is a tendency to use only the
top categories while in others there is a tendency to use only the bottom
categories.

In a decision making set up when the judges provide scores, one can generate
{\em implied pairwise comparisons} that reflect the intensity of the
preference.  This is done by letting this intensity be the difference in the
scores for the two respective objects (these are called {\em additive}
comparisons further discussed later).  \citet{HL06} demonstrated that an
aggregate rating that minimizes the penalties for differing from the individual
judges' implied {\em pairwise comparisons} overcomes the issue of using
different parts of the scale and is less sensitive to subjective scales than
the use of cardinal scores alone.  These type of penalties are called {\em
separation penalties}, and the optimization problem that seeks to assign scores
that minimize the total separation penalties is called the {\em separation
problem}. The separation-deviation (SD) model, proposed in
\citep{Hoc04,Hoc06,HL06}, considers an aggregate rating scenario where the
input to the rating process is given as \emph{separation gaps} and
\emph{point-wise scores}. A separation gap is a quantity that expresses the
intensity of the preference of one object $i$ over another $j$ by one
particular judge. A point-wise score is a cardinal score of an object. The SD
optimization problem combines the objective of minimizing the penalties of the
deviation of the assigned scores to the point-wise scores assigned by judges to
each object and the minimization of the separation penalties. For any choice of
penalty functions the aggregate rating obtained by solving the SD model is a
complete-rating that minimizes the sum of penalties on deviating from the given
point-wise scores and separation gaps. The SD model is solved in polynomial
time if the penalty functions are convex.  It is NP-hard otherwise.

%

In our problem setting the judges provided only point-wise scores.  Therefore
there are no pairwise comparisons provided directly.  Instead we use here the
pairwise comparisons {\em implied} by the scores.  The mechanism we propose
here uses the SD model for both the rankings and ratings provided by the
judges.  For the rankings the penalty functions proposed are not convex.  We
``convexify" those functions and attain an optimization model that {\em
combines} the separation and deviation penalties for deviating from the
rankings and from the ratings of all judges. This is the first aggregate
decision model that combines both ordinal and cardinal inputs.


The advantages of the mechanism proposed are obvious in comparison to standard
approaches.  It is easy to recognize a discrepancy in scores given to the same
object by different judges.  However, it is possible that the scores given are
very close, yet each one is assigned from a different subjective scale.  For
one judge the score of 7 out of 10 can indicate the top evaluation, whereas for
another it may mean the very bottom.  Such scale differences cannot be
identified by considering the variance of the scores alone.  Our optimal
solution to the SD problem, with the given penalty functions, allow to identify
immediately the largest penalty pairs which, if large enough, indicate that
different judges disagreed significantly on the comparison between such pairs
of objects.  This permits to identify inconsistencies and outliers that could
be judges who are too lenient or too strict, or for other reasons had intensity
of preference substantially different from the others.  As such the methodology
proposed not only provides an aggregate ranking, but also clarifies the
disagreements and inconsistencies that allow to go back and possibly
investigate the reasons for those outliers.



The paper is organized as follows: Section \ref{sec:literature} provides a
literature review on some relevant aggregate group-decision making techniques
for rankings and ratings aggregation. Section \ref{sec:data} describes the
evaluation methodology used in the 2007 MSOM's SPC, and gives examples where
the differences in scale used by the judges are evident. Section
\ref{sec:method} reviews the models and distance metrics used to construct the
penalty functions and defines the notions of consensus ranking and consensus
rating used here. Section \ref{sec:jointAggregation} describes the methodology
for the combined use of the given ratings and rankings in order to obtain the
aggregate ranting-ranking pair. Section \ref{sec:results} uses the methodology
presented in Section \ref{sec:method} to rank the contestants in the 2007
MSOM's SPC and analyzes the obtained results. Finally, Section
\ref{sec:conclusion} provides comments on our group-decision making framework
and its usefulness for different applications and decision-making scenarios.


\section{Literature Review}\label{sec:literature}
The ranking aggregation problem has been studied extensively, especially in the
social choice literature. In this context, one of the most celebrated results
is Arrows's impossibility theorem \citep{Arr63}, which states that there is no
``satisfactory'' method to aggregate a set of rankings. Kenneth Arrow defined a
satisfactory method as one that satisfies the following properties: universal
domain, no imposition, monotonicity, independence of irrelevant alternatives,
and non-dictatorship.

\citet{KS62} proposed a set of axioms that a distance metric between two
complete rankings should satisfy. They proved that these axioms were jointly
satisfied by a unique metric distance. This distance between two rankings is
measured by the number of {\em rank reversals} between them. A rank reversal is
incurred whenever two objects have a different relative order in the given
rankings. Similarly, {\em half} a rank reversal is incurred whenever two
objects are tied in one ranking but not in the other. Kemeny and Snell defined
the consensus ranking as the ranking that minimizes the sum of the distances to
each of the input rankings. \citet{BTT89} showed that the optimization problem
that needs to be solved to find the Kemeny-Snell consensus ranking is NP-hard.

Following the work of Kemeny-Snell, several axiomatic approaches have been
developed to determine consensus. For instance, \citet{Bog73} developed an
axiomatic distance between partial orders. One of the applications of Bogart's
distance is to determine a consensus partial order from a set of partial
orders. \citet{MC10} developed an axiomatic distance
between incomplete rankings that is used here. 

The difficulties presented by Arrow's impossibility theorem and the NP-hardness
of finding the Kemeny-Snell's consensus ranking can be overcome by replacing
ordinal rankings by (cardinal) ratings. Following this direction, \citet{Kee76}
proved that the \emph{averaging method} satisfied all of Arrow's desirable
properties. In the averaging method, the consensus rating of each object is the
average of the scores it received. The most immediate drawback of this approach
is that the averaging method implicitly requires that all judges use the same
rating scale; that is, that all individuals are equally strict or equally
lenient in their score assignments.  This work also ignores the aspect of
pairwise comparisons, which is essential to the Kemeny-Snell model.

Pairwise comparisons intensities are the input to Saaty's Analytic Hierarchy
Process technique \citep{Saa77}.  There, the optimal scores are found by the
principal eigenvector technique.  The readers are referred to \citep{Hoc10} for
an analysis of the principal eigenvector method in the context of aggregate
decision making.

The separation-deviation model of \citep{Hoc04,Hoc06,HL06} addresses the
computational shortcomings of the Kemeny-Snell model, and the decision quality
inadequacies of the principal eigenvector method. This model takes point-wise
scores and potentially also pairwise comparison as inputs.  It is the building
block of the mechanism proposed here.  As pointed out above, the respective
separation-deviation optimization problem is solvable in polynomial time if all
the penalty functions are convex \citep{HL06}.


The rating aggregation problem has also been studied in the context of
multi-criteria decision making literature. \citet{HL06} showed the equivalence
between the rating aggregation problem and the multi-criteria decision making
problem. In this context, the non-axiomatic ELECTRE \citep{BRV75} and PROMETHEE
\citep{BV85} methods (and their extensions) solve the rating aggregation
problem that arises from a multi-criteria decision problem by transforming it
in some sense to a ranking aggregation problem. This transformation is claimed
to be needed because each criterion is evaluated on a different scale.

\section{The Data}\label{sec:data}
The data used here is the evaluations for the 2007 MSOM's SPC. There were 58
papers submitted to the competition and 63 judges participated in the
evaluation process. Each of the 63 judges evaluated only three to five out of
the 58 papers; and each of the 58 papers was evaluated by only three to five
out of the 63 judges. Each judge reviewed and evaluated the assigned papers on
the attributes (scale):


\renewcommand{\labelenumi}{\Alph{enumi})}
\begin{compenum}
\item Problem importance/interest (1--10),
\item Problem modeling (0--10),
\item Analytical results (0--10),
\item Computational results (0--10),
\item Paper writing (1--10), and
\item Overall contribution to the field (Field contribution, for short)
    (1--10).
\end{compenum}
\renewcommand{\labelenumi}{\arabic{enumi}.}
On each attribute, the judges assigned scores according to the score guidelines
provided (see Table \ref{tab:interpretations}). In addition, each judge also
provided an ordinal ranking of the papers he/she reviewed (1 = best, 2 = second
best, etc.).

\begin{table}[H]
\centering \caption{Interpretation of each numerical score. The journals
considered are: MSOM, Operations Research (OR) and Management Science
(MS).}\label{tab:interpretations} \footnotesize\begin{tabular}{c|l} {\bf Score}
& {\bf Definition / Interpretation} \\ \hline
10 & Attribute considered is comparable to that of the best papers published in the journals. \\
8,9 & Attribute considered is comparable to that of the average papers published in the journals. \\
7 & Attribute considered is at the minimum level for publication in the journals. \\
5,6 & Attribute considered independently would require a minor revision before publication in the journals. \\
3,4 & Attribute considered independently would require a major revision before publication in the journals. \\
1,2 & Attribute considered would warrant by itself a rejection if the paper were submitted to the journals. \\
0 & Attribute considered is not relevant or applicable to the paper being
evaluated.
\end{tabular}
\end{table}

Although precise score interpretations were provided to the judges (Table
\ref{tab:interpretations}), they nevertheless appeared to have differed
significantly in their evaluation and must have interpreted the scale
differently.  Examples of this phenomenon are illustrated for paper 43, in
Table \ref{tab:notcomp43}, and paper 26, in Table \ref{tab:notcomp26}. To
maintain the anonymity of judges and papers the judge and paper identification
numbers were assigned randomly.
\begin{table}[H]
\centering \caption{\small Evaluations received on paper
43.}\label{tab:notcomp43} \tiny\begin{tabular}{c|cccccccc}
 {\bf Judge}  & {\bf Problem} & {\bf Problem} & {\bf Analytical} & {\bf Computational} & {\bf Paper} & {\bf Field} & {\bf Paper} \\
 {\bf      }  & {\bf Importance} & {\bf Modeling} & {\bf Results} & {\bf Results} & {\bf Writing} & {\bf Contribution} & {\bf Ranking} \\ \hline
        47    &          9 &          8 &          8 &          8 &          9 &          9 &          1 \\
      \nn6    &          6 &          4 &          2 &          4 &          4 &        4.5 &          1 \\
        55    &          9 &          6 &          0 &          9 &          8 &          6 &          2 \\
      \nn2    &          7 &          7 &          2 &          6 &        7.5 &          4 &          3
\end{tabular}
\end{table}

A detailed examination of Table \ref{tab:notcomp43} illustrates that paper 43
received in the Problem Modeling category a score of 8 by one judge (meaning
that the Problem Modeling in the paper is comparable to that in an average
paper published in MSOM, OR and MS), and a score of 4 by other judge (meaning
that the problem modeling in the paper requires a major revision before
publication in MSOM, OR and MS). These score differences are not insignificant.
Another example of the differences between the judges' evaluations is found on
the Analytical Results category. In this category, a judge gave a score of 8
(meaning that the analytical results in the paper are comparable to those in an
average paper published in MSOM, OR and MS), two judges gave a score of 2
(meaning that the analytical results in the paper are so bad that the paper
should be rejected by MSOM, OR and MS), and the remaining judge considered that
the category was not applicable to the paper (thus assigned the value of zero).

\begin{table}[H]
\centering \caption{\small Evaluations received on paper
26.}\label{tab:notcomp26} \tiny\begin{tabular}{c|cccccccc}
 {\bf Judge}  & {\bf Problem} & {\bf Problem} & {\bf Analytical} & {\bf Computational} & {\bf Paper} & {\bf Field} & {\bf Paper} \\
 {\bf      }  & {\bf Importance} & {\bf Modeling} & {\bf Results} & {\bf Results} & {\bf Writing} & {\bf Contribution} & {\bf Ranking} \\ \hline
        21    &          8 &         10 &          8 &          8 &          5 &          8 &          3 \\
        24    &          8 &          9 &          8 &         10 &          7 &          8 &          1 \\
        14    &          7 &          2 &          3 &          2 &          2 &          2 &          5 \\
        26    &          8 &          8 &          7 &          8 &          8 &          7 &          3 \\
        49    &         10 &          7 &          6 &          9 &          9 &          8 &          1
\end{tabular}
\end{table}
In Table \ref{tab:notcomp26} the data shows that judge 14's evaluations were
not on the same scale as the evaluations of the other judges. In particular,
in all attributes (with the exception of Problem Importance) judge 14 gave a
score indicating that the paper would be rejected by MSOM, OR and MS; on the
other hand in every attribute all of the other judges considered the paper is
worth of publishing (some of their evaluations even indicate that the paper
would be among the best papers published in MSOM, OR and MS!). Such
discrepancies in the judges' evaluations are quite common throughout the data.

Henceforth, we use as the input point-wise ratings the (cardinal) scores only
on the attribute ``Overall Contribution to the Field'' (``Field Contribution'',
for short). This is because the authors and the head judge of the 2007 MSOM's
SPC, believe that, among all the attributes that were scored according to the
cardinal scale in Table \ref{tab:interpretations} (i.e., excluding the ordinal
paper ranking), this attribute is the single most important attribute
evaluated.

\section{Preliminaries}\label{sec:method}
This section gives the notation used throughout the rest of the paper, and
reviews the concepts of: the separation-deviation (SD) model, distance between
incomplete ratings, and distance between incomplete rankings.

\subsection{Notation}
Let $V$ be the ground set of $n$ objects to be rated; without loss of
generality, we assign a unique identifier to each element so that
$V=\set{1,2,...,n}$.  The judges are $K$ individuals. Each judge $k$,
$k\in\set{1,2,...,K}$, provides a set of scores, or ratings vector, $\a^{(k)}$
of the objects in a subset $\sA^{(k)}$ of $V$.  Thus $a^{(k)}_j$ is the score
of object $j$ by the k\th individual, and $a^{(k)}_j$ is undefined if the k\th
individual did not rate object $j$.  Without loss of generality, we assume that
the scores are integers contained in a pre-specified interval $[\ell,u]$. The
\emph{range} of the ratings is defined as $R\equiv u-\ell$.

We say that judge's $k$ {\em implied pairwise comparison}, or {\em separation
gap} of $i$ to $j$ is $p^{(k)}_{ij}$ where

\[p^{(k)}_{ij}=\begin{cases}
a^{(k)}_i-a^{(k)}_j & \text{if }i\in \sA^{(k)} \text{ and } j\in \sA^{(k)} \\
\text{undefined} & \text{otherwise.}\end{cases}\]


Analogously, in the ordinal setting of the incomplete-ranking aggregation
problem, each judge $k$ provides an incomplete ranking $\b^{(k)}$ of the
objects in $\sB^{(k)}$, a subset of $V$. Here $b^{(k)}_i$ is the rank (an
ordinal number) of object $i$ in the ranking provided by the k\th individual,
and $b^{(k)}_i$ is undefined if individual $k$ did not rank object $i$.

The implied separation gaps for ordinal rankings are
$\sign(b^{(k)}_i-b^{(k)}_j)$ for $i,j\in \sB^{(k)}$, where the $\sign$ function
is defined as:
\[\sign (x)=\begin{cases}
-1 & \text{if } x<0 \\
0  & \text{if } x=0 \\
1 & \text{if } x>0.
\end{cases}\]

For a vector of scores, or ratings, $\a$ of a set of objects, we denote by
$\rank{\a}$ the complete ranking of those objects obtained by sorting the
objects according to their scores in $\a$.  For example, the vector of scores
$(4.5, 5, 3, 2.7, 3)$ corresponds to the ranking $(2,1,3,5,3)$.

\subsection{Review of the Separation-Deviation Model}\label{sec:sepdev}
The SD model can be applied to group-decision making problems where the input
is given as pairwise comparisons and/or point-wise scores.
In the model formulation, the variable $x_i$ is the aggregate score of the i\th
object, and the variable $z_{ij}$ is the aggregate separation gap of the i\th
over the j\th object.  The separation gaps must be consistent.  A set of
separation gaps, $p_{ij}$, is said to be {\em consistent} if and only if for
all triplets $i,j,k$, $p_{ij}+p_{jk}=p_{ik}$. In \citep{Hoc10,HL06} it was
proved that the consistency of a set of separation gaps is equivalent to the
existence of a set of scores $\o_i$ for $i=1,\dots,n$ so that
$p_{ij}=\o_i-\o_j$.

 The mathematical programming formulation of the
SD model is:
\begin{subequations}\label{eqn:sepdev}
\begin{optprog}
(SD)$\qquad$ \optaction[\x,\z]{min} & \objective{\sum_{k=1}^K \sum_{i=1}^n \sum_{j=1}^n f^{(k)}_{ij}(z_{ij}-p^{(k)}_{ij}) + \sum_{k=1}^K \sum_{i=1}^n g^{(k)}_i(x_i-a^{(k)}_i)} \\
subject to & z_{ij}&  =   & x_i-x_j    & i=1,\dots,n; & j=1,\dots,n \label{eqn:c1} \\
           &  \ell & \leq & x_i \leq u & i=1,\dots,n  \\
           &  x_i  & \in  & \Z         & i=1,\dots,n.
\end{optprog}
\end{subequations}
The function $f^{(k)}_{ij}(\cdot)$ penalizes the difference between the
aggregate separation gap of the object pair $(i,j)$ and the k\th reviewer's
separation gap of the object pair $(i,j)$. The function $g^{(k)}_{i}(\cdot)$
penalizes the difference between the aggregate score of object $i$ and the k\th
reviewer's score of object $i$. In order to ensure polynomial-time solvability,
the functions $f^{(k)}_{ij}(\cdot)$ and $g^{(k)}_i(\cdot)$ must be
\textbf{convex}. In the context of rating aggregation, the penalty functions
assume the value $0$ for the argument $0$; meaning that if the output
separation gap for $i$ an $j$, $z_{ij}$ agree with $p^{(k)}_{ij}$ then
$f^{(k)}_{ij}(z_{ij}-p^{(k)}_{ij})=$ $f^{(k)}_{ij}(0)=$ $0$. If $i\not\in
\sB^{(k)}$, then $g^{(k)}_i(\dot)$ is set to the constant function $0$;
similarly, if at least one of $i$ or $j$ $\not\in \sB^{(k)}$, then
$f^{(k)}_{ij}(\dot)$ is set to the constant function $0$. Constraints
\eqref{eqn:c1} enforce the consistency of the aggregate separation gaps
conforming to the aggregate rating.

It was proved in \citep{Hoc04,Hoc06,HL06} that problem (SD) is a special case
of the convex dual of the minimum cost network flow (CDMCNF) problem. The most
efficient algorithm known for the CDMCNF has a running time of
$O(mn\log\frac{n^2}{m}\log(u-\ell))$ \citep{AHO03}, where $m$ is the total
number of given separation gaps, and $n$ is the number of objects.
\citet{AHO04} presented an alternative algorithm that uses a minimum-cut
algorithm as a subroutine.


\subsection{Distance between Incomplete-Ratings}\label{sec:IncompleteRatingAggregation}
Defining a penalty function on separation gaps is equivalent to quantifying the
distance between them. \citet{CK85} proposed a distance between complete
ratings. This distance function was adapted to incomplete ratings in
\citep{MC10}. It was shown that for a set of desirable properties this
adaption, called \textit{normalized projected Cook-Kress distance} (NPCK), is
the only one that satisfies all those properties.

Given two incomplete ratings $\a^{(1)}$ and $\a^{(2)}$, the NPCK distance
between the implied separation gaps is
\begin{equation}\label{eqn:npck}
d_{NPCK}(\a^{(1)},\a^{(2)})=\mathcal{C}\sum_{i\in \sA^{(1)}\bigcap \sA^{(2)}}\sum_{j\in \sA^{(1)}\bigcap \sA^{(2)}}
|p^{(1)}_{ij}-p^{(2)}_{ij}|,
\end{equation}
where
\begin{equation}\label{eqn:CteRate}
\mathcal{C}=\left(4\cdot R\cdot \ceil{\frac{\abs{\sA^{(1)}\bigcap
\sA^{(2)}}}{2}}\cdot\floor{\frac{\abs{\sA^{(1)}\bigcap \sA^{(2)}}}{2}}\right)^{-1}.
\end{equation}
$\mathcal{C}$ is a normalization constant that guarantees that $0\leq
d_{NPCK}(\cdot,\cdot)\leq 1$ and $R$ is the \emph{range} of the ratings,
$R\equiv u-\ell$. We note that $d_{NPCK}(\a^{(1)},\a^{(2)})=0$ indicates a
total agreement between the ratings $\a^{(1)}$ and $\a^{(2)}$, and
$d_{NPCK}(\a^{(1)},\a^{(2)})=1$ indicates a total disagreement between the
ratings $\a^{(1)}$ and $\a^{(2)}$. The normalization is important so that the
distances in problem \eqref{eqn:minnpck} are comparable to each other even when
the individuals rate a different number of objects. The normalization constant
$\mathcal{C}$ was chosen to address the following difficulties: (a) Each of the
distances in problem \eqref{eqn:minnpck} are between a complete rating $\xc$
and an incomplete rating. (b) The number of objects rated by each incomplete
rating are different; therefore the distances in problem \eqref{eqn:minnpck}
are over different dimensional spaces (the distance only considers the objects
rated by the incomplete rating). (c) Distances in higher dimensional spaces
tend to be bigger than distances in lower dimensional spaces; specifically,
observe that the number of summands in equation \eqref{eqn:npck} is the square
of the number of objects rated by the incomplete rating.

In \citep{MC10} the \textit{consensus rating}, $\xc$, is the optimal solution
to the following optimization problem:
\begin{equation}\label{eqn:minnpck}
\min_{\x} \sum_{k=1}^K d_{NPCK}(\a^{(k)},\x).
\end{equation}

The problem of finding the consensus rating is as a special case of the SD
model and therefore solvable in polynomial time.

\subsection{Distance between Incomplete-Rankings}\label{sec:IncompleteRankingAggregation}
Given a set of incomplete rankings, $\set{\b^{(k)}}_{k=1}^K$, the consensus
ranking is defined as the complete ranking closest to the given incomplete
rankings. \citet{KS62} proposed a distance between complete rankings. This
distance function was adapted to incomplete rankings in \citep{MC10}. It was
shown that for a set of desirable properties this adaption, called
\textit{normalized projected Kemeny-Snell distance} (NPKS), is the only one
that satisfies all those properties.

Given two incomplete rankings $\b^{(1)}$ and $\b^{(2)}$, the NPKS distance
between them is calculated as follows:
\begin{equation}\label{eqn:npks}
d_{NPKS}(\b^{(1)},\b^{(2)})=\mathcal{D}\sum_{i\in \sB^{(1)}\bigcap \sB^{(2)}}\sum_{j\in \sB^{(1)}\bigcap \sB^{(2)}}
\frac{1}{2}|\sign(b^{(1)}_i-b^{(1)}_j)-\sign(b^{(2)}_i-b^{(2)}_j)|,
\end{equation}
where $\mathcal{D}=\left(\abs{\sB^{(1)}\bigcap
\sB^{(2)}}^{2}-\abs{\sB^{(1)}\bigcap \sB^{(2)}}\right)^{-1}$. \rem{ERICK: This
is correct and not 'choose 2' because of the double sum instead of sum on
objects pairs. (The $\frac{1}{2}$ is needed to count each reversal as 1 and
each half reversal as 1/2).} $\mathcal{D}$ is a normalization constant that
guarantees that $0\leq d_{NPKS}(\cdot,\cdot)\leq 1$. When
$d_{NPKS}(\b^{(1)},\b^{(2)})=0$ there is a total agreement between $\b^{(1)}$
and $\b^{(2)}$, and when $d_{NPKS}(\b^{(1)},\b^{(2)})=1$ there is a total
disagreement between $\b^{(1)}$ and $\b^{(2)}$. The normalization is important
so that the distances in problem \eqref{eqn:minnpks} are comparable to each
other even when the individuals rank a different number of objects. The
normalization constant $\mathcal{D}$ was chosen to address the following
difficulties: (a) Each of the distances in problem \eqref{eqn:minnpks} are
between a complete ranking $\xo$ and an incomplete ranking. (b) The number of
objects ranked by each incomplete ranking are different; therefore the
distances in problem \eqref{eqn:minnpks} are over different dimensional spaces
(the distance only considers the objects ranked by the incomplete ranking). (c)
Distances in higher dimensional spaces tend to be bigger than distances in
lower dimensional spaces; specifically, observe that the number of summands in
equation \eqref{eqn:npks} is the square of the number of objects ranked by the
incomplete ranking.

The distance $d_{NPKS}(\b^{(1)},\b^{(2)})$  has the following natural
interpretation: The distance between two incomplete rankings is proportional to
the number of {\em rank reversals} between them. Where a rank reversal is
incurred whenever two objects have a different relative order in the rankings
$\b^{(1)}$ and $\b^{(2)}$. Similarly, \textit{half} a rank reversal is incurred
whenever two objects are tied in one ranking but not in the other ranking.

In \citep{MC10} the \textit{consensus ranking}, $\xo$, is the optimal solution
to
\begin{equation}\label{eqn:minnpks}
\min_{\x} \sum_{k=1}^K d_{NPKS}(\b^{(k)},\x).
\end{equation}


\subsection{Convexifying the Rankings Penalty Function}

In contrast to problem \eqref{eqn:minnpck}, problem \eqref{eqn:minnpks} is
NP-hard. We propose here to convexify the nonlinear sign functions in
$d_{NPKS}(\cdot,\cdot)$ as suggested in \citep{MC10}:

\begin{equation} \label{eqn:sepkemrelaxH}
    h^{(k)}_{ij}(z_{ij}) = \begin{cases}
                 \max\set{0,\frac{z_{ij}+1}{2}} & \text{if sign$(b^{(k)}_i-b^{(k)}_j)=-1$} \\
                 \max\set{\frac{-z_{ij}}{2},\frac{z_{ij}}{2}} & \text{if sign$(b^{(k)}_i-b^{(k)}_j)=0$} \\
                 \max\set{\frac{1-z_{ij}}{2},0} & \text{if sign$(b^{(k)}_i-b^{(k)}_j)=1$}
 \end{cases}
\end{equation}

The following formulation is then a convex version of problem
\eqref{eqn:minnpks}:
\begin{subequations}\label{eqn:sepkemrelax}
\begin{optprog}
\optaction[\x,\z]{min} & \objective{\sum_{k=1}^{K}\ \mathcal{D}_k\sum_{i\in \sB^{(k)}}\ \sum_{j\in \sB^{(k)}}\ h^{(k)}_{ij}(z_{ij})} \label{eqn:sepkemrelax_obj} \\
subject to & z_{ij} & = & x_i-x_j & i=1,\dots,n; & j=1,\dots,n\text{.} 
\rem{ERICK: This is correct according to derivation in green research notebook}
\end{optprog}
\end{subequations}


We conclude this section by observing that, for the rankings given by the
judges in the 2007 MSOM's SPC, the optimal solution to convexified problem
\eqref{eqn:sepkemrelax} is a good approximation to the optimal solution of
problem \eqref{eqn:minnpks}. That is, the distance $d_{NPKS}(\cdot,\cdot)$
between the optimal solution to problem \eqref{eqn:minnpks}, (obtained using
the implicit hitting set approach problem of \citealt{KM}), and the optimal
solution to the convex approximation, problem \eqref{eqn:sepkemrelax}, is only
$0.1606$. This is further discussed in Section \ref{sec:results}.

\section{Joint Aggregation of Ratings and Rankings}\label{sec:jointAggregation}
This section describes the model to jointly aggregate the ratings and the
rankings. The goal of this model is not only to fairly represent the judges'
rating and the judges' rankings, but also to balance the cardinal and ordinal
evaluations.  We refer to this optimization model as the \textit{Combined
Aggregate raTing} problem, or (CAT).

The input to (CAT) is a set of ratings $\set{\a^{(k)}}_{k=1}^K$ and a set of
rankings $\set{\b^{(k)}}_{k=1}^K$. (CAT) is a combination of the rating
aggregation problem \eqref{eqn:minnpck} and the ranking aggregation problem
\eqref{eqn:minnpks}. In order to guarantee that ratings rankings weigh equally
in the optimization problem (CAT), both distance functions, $d_{NPCK}$ and
$d_{NPKS}$, are normalized. Note that one can weigh these distances differently
if justified by the circumstances of the decision context. Also, the choice of
$d_{NPCK}$ and $d_{NPKS}$ as penalty functions, or distances, can be replaced
by other distances between incomplete ratings and between incomplete rankings,
respectively.

\begin{optprog}
${\mbox {\rm (CAT)}}\qquad$ \optaction[\x]{min} & \objective{\sum_{k=1}^K
d_{NPCK}\Bigl(\a^{(k)},\x\Bigr)+\sum_{k=1}^K
d_{NPKS}\Bigl(\b^{(k)},\rank{\x}\Bigr)}
\end{optprog}

We next establish that (CAT) is NP-hard by reducing problem \eqref{eqn:minnpks}
(which is NP-hard) to it.

\begin{lemma}
(CAT) is NP-hard.
\end{lemma}
\begin{proof}
Given an instance of problem \eqref{eqn:minnpks}, a set of incomplete rankings
$\set{\b^{(k)}}_{k=1}^K$), one can transform it (in polynomial time) to an
instance of (CAT) as follows. Keep unchanged $\set{\b^{(k)}}_{k=1}^K$, and
create a set of ratings $\set{\a^{(k)}}_{k=1}^K$ such that each rating
evaluates exactly one object (the choice of object is irrelevant; moreover all
of the ratings can evaluate the same object). From the definition of $d_{NPCK}$
(equation \eqref{eqn:npck}), it follows that, for every $\x$, the first summand
in (CAT) will be equal to $0$. Therefore, with this choice of ratings,
$\rank{\x^*}$, where $\x^*$ is the optimal solution to (CAT), will be the
optimal solution to problem \eqref{eqn:minnpks}.
\end{proof}

The (nonlinear, nonconvex) mathematical programming formulation of (CAT) is
\begin{subequations}\label{eqn:simultnonlinear}
\begin{optprog}
\optaction[\x,\z]{min} & \objective{\sum_{k=1}^{K}\ \mathcal{C}_k\sum_{i\in \sA^{(k)}}\ \sum_{j\in \sA^{(k)}}\ \abs{z_{ij}-p^{(k)}_{ij}} + }\nonumber \\
                         & \objective{\sum_{k=1}^{K}\ \mathcal{D}_k\sum_{i\in \sB^{(k)}}\ \sum_{j\in \sB^{(k)}}\ \frac{1}{2}|\sign(z_{ij})-\sign(b^{(k)}_j-b^{(k)}_i)|  } \label{eqn:nlobj} \\
subject to & z_{ij}     &   =  & x_i-x_j & i=1,\dots,n; & j=1,\dots,n \label{eqn:nl1} \\
           &  \ell      & \leq & x_i \leq u        & i=1,\dots,n \label{eqn:nlbounds} \\
           &  x_i &  \in & \Z                      & i=1,\dots,n. \label{eqn:nlintegraliy}
\end{optprog}
\end{subequations}

The convexification of the objective of problem \eqref{eqn:simultnonlinear}, as
described in Section \ref{sec:IncompleteRankingAggregation}, results in the
convex formulation:

\begin{subequations}\label{eqn:simultconvex}
\begin{optprog}
\optaction[\x,\z]{min} & \objective{\sum_{k=1}^{K}\ \mathcal{C}_k\sum_{i\in
\sA^{(k)}}\
\sum_{j\in \sA^{(k)}}\ \abs{z_{ij}-p^{(k)}_{ij}} + \sum_{k=1}^{K}\ \mathcal{D}_k\sum_{i\in \sB^{(k)}}\ \sum_{j\in \sB^{(k)}}\ h^{(k)}_{ij}(z_{ij})  } \label{eqn:scobj} \\
subject to & z_{ij}     &   =  & x_{i}-x_{j} & i=1,\dots,n; & j=1,\dots,n \label{eqn:sc1} \\
           &  \ell      & \leq & x_i \leq u        & i=1,\dots,n \label{eqn:scbounds} \\
           &  x_i &  \in & \Z                      & i=1,\dots,n \label{eqn:scintegraliy} \\
where, & \objective{
    h^{(k)}_{ij}(z_{ij}) = \begin{cases}
                 \max\set{0,\frac{z_{ij}+1}{2}} & \text{if sign$(b^{(k)}_j-b^{(k)}_i)=-1$} \\
                 \max\set{\frac{-z_{ij}}{2},\frac{z_{ij}}{2}} & \text{if sign$(b^{(k)}_j-b^{(k)}_i)=0$} \\
                 \max\set{\frac{1-z_{ij}}{2},0} & \text{if sign$(b^{(k)}_j-b^{(k)}_i)=1$.}
               \end{cases}}\label{eqn:scH}
\rem{ERICK: This is correct according to derivation in green research notebook}
\end{optprog}
\end{subequations}
Problem \eqref{eqn:simultconvex} is a special case of the convex SD model and
thus solvable in polynomial time.

\textbf{Remark: } Note that in equations \eqref{eqn:nlobj} and \eqref{eqn:scH},
the argument of the $\sign$ function is $b^{(k)}_j-b^{(k)}_i$ and not
$b^{(k)}_i-b^{(k)}_j$ as in equations \eqref{eqn:npks} and
\eqref{eqn:sepkemrelaxH}. This is because of the classical convention that in
the given ratings \textbf{high cardinal numbers} are assigned to the
\textbf{most} preferred objects; while in the given rankings \textbf{high
ordinal numbers} are assigned to the \textbf{least} preferred objects.

The optimal solution to (CAT) is a combined aggregate rating-ranking pair which
is denoted by $\xdc$, and its implied ranking is denoted by $\xdo$.


Next, we propose two mechanisms to identify inconsistencies in the given
evaluations (e.g. outliers, judges that are too lenient or too strict, etc.).
This information is helpful so that (say) the lead decision maker initiates an
investigation of the nature of the discrepancies and acts appropriately (for
example, by discussing these inconsistencies with the judges and promote a
discussion with the objective of alleviating them).

The first mechanism is to use the solution $\xdc$ to identify (a) judges whose
evaluations differ the most with the rest of the evaluations and (b) objects
such that the judges evaluating them had particularly divergent evaluations.
These judges (objects) are those that assigned (received) scores that disagree
the most with $\xdc$. Specifically, we use the separation penalty to identify
the judges whose evaluations are at the farthest distance from $\xdc$ (i.e.,
have the highest separation penalty). Specifically, the contribution of judge
$k$ to the separation penalty is
\begin{equation}
\sum_{i\in \sA^{(k)}} \sum_{j\in \sA^{(k)}} \mathcal{C}_k\abs{(\xdc_i-\xdc_j)-(\a^{(k)}_i-\a^{(k)}_j)}.
\end{equation}
Similarly, we use the separation penalty to identify the objects such that the
judges evaluating them had particularly divergent evaluations. These objects
are those with the highest contribution to the separation penalty. The
contribution of object $i$ to the separation penalty is
\begin{equation}
\sum_{k|i\in \sA^{(k)}} \sum_{j\in \sA^{(k)}} \mathcal{C}_k\abs{(\xdc_i-\xdc_j)-(\a^{(k)}_i-\a^{(k)}_j)}.
\end{equation}

The second mechanism to identify inconsistencies in the given evaluations is
based on Brans and Vincke's PROMETHEE method \citep{BV85}. The mechanism is to
aggregate the consensus rating $\xc$ (solution to problem \eqref{eqn:minnpck})
and the consensus rating $\xo$ (solution to problem \eqref{eqn:minnpks}) into a
partial order $(P,T,I)$ as follows:
\begin{subequations}\label{eqn:partial}
\begin{align}
& a \text{ is preferred to } b\ (a\ P\ b) && \text{ if } && \begin{cases}
\xc(a)>\xc(b) \text{ and } \xo(a)\geq\xo(b)\\
\xc(a)\geq\xc(b) \text{ and } \xo(a)>\xo(b)\end{cases}\\
& a \text{ and } b \text{ are tied } (a\ T\ b) && \text{ if } && \xc(a)=\xc(b) \text{ and } \xo(a)=\xo(b)\\
& a \text{ and } b \text{ are incomparable } (a\ I\ b) && \text{  } && \text{otherwise.}
\end{align}
\end{subequations}
Thus, by construction, the partial order $(P,T,I)$ summarizes the agreement (or
lack thereof) between the consensus rating $\xc$ and the consensus ranking
$\xo$.

Section \ref{sec:results} illustrates these mechanisms and their usefulness for
identifying objects whose evaluations deserve special attention/further
discussion.

\section{Results}\label{sec:results}
We illustrate here how to use the proposed mechanism in the ranking of the
contestants of the 2007 MSOM's SPC. These results are compared to those
obtained by aggregating only the cardinal evaluations, and those obtained by
aggregating only the ordinal evaluations.

Table \ref{tab:results} gives the consensus rating (optimal solution to problem
\eqref{eqn:minnpck}) $\xc$; the (approximate) consensus ranking (optimal
solution to problem \eqref{eqn:sepkemrelax}) $\xo$; and, the combined aggregate
rating $\xdc$ and ranking $\xdo$ (optimal solutions to problem
\eqref{eqn:simultconvex}).

\begin{table}[H]
\centering \caption{Aggregate ratings and rankings for the 2007 MSOM's
SPC.}\label{tab:results}
\tiny\begin{tabular}{r|rrrrrr|rrrr}%
{\bf Paper} & {\bf $\xc$} & {\bf $\xo$} & {\bf $\xdc$} & {\bf $\xdo$} &
& {\bf Paper} & {\bf $\xc$} & {\bf $\xo$} & {\bf $\xdc$} & {\bf
$\xdo$}\\\cline{1-5}\cline{7-11}
         1 &          3 &         41 &          4 &         41 &            &         30 &          5 &         24 &          5 &         23 \\
         2 &        5.5 &         24 &          5 &         23 &            &         31 &          6 &         24 &          5 &         23 \\
         3 &          5 &         41 &          4 &         41 &            &         32 &        5.5 &         24 &          5 &         23 \\
         4 &        5.5 &         24 &          5 &         23 &            &         33 &        5.5 &         41 &          4 &         41 \\
         5 &        4.5 &         41 &          4 &         41 &            &         34 &        6.5 &          3 &          7 &          2 \\
         6 &        6.5 &          9 &          6 &          8 &            &         35 &          5 &         41 &          4 &         41 \\
         7 &        6.5 &         24 &          5 &         23 &            &         36 &        5.5 &         24 &          5 &         23 \\
         8 &          6 &          9 &          6 &          8 &            &         37 &          5 &         24 &          5 &         23 \\
         9 &        5.5 &          9 &          6 &          8 &            &         38 &          3 &         41 &          4 &         41 \\
        10 &          6 &          3 &          7 &          2 &            &         39 &          7 &          9 &          6 &          8 \\
        11 &          6 &          3 &          7 &          2 &            &         40 &          5 &         41 &          4 &         41 \\
        12 &          5 &         24 &          5 &         23 &            &         41 &        6.5 &          3 &          6 &          8 \\
        13 &        6.5 &          9 &          6 &          8 &            &         42 &          6 &         24 &          5 &         23 \\
        14 &        7.5 &          9 &          6 &          8 &            &         43 &          6 &          9 &          6 &          8 \\
        15 &          5 &         41 &          4 &         41 &            &         44 &        4.5 &          9 &          5 &         23 \\
        16 &          4 &         53 &          3 &         53 &            &         45 &        5.5 &         24 &          5 &         23 \\
        17 &        6.5 &          9 &          6 &          8 &            &         46 &          6 &          9 &          6 &          8 \\
        18 &        3.5 &         53 &          3 &         53 &            &         47 &          6 &          9 &          6 &          8 \\
        19 &        5.5 &         24 &          5 &         23 &            &         48 &        6.5 &          9 &          6 &          8 \\
        20 &        2.5 &         53 &          2 &         58 &            &         49 &        7.5 &          2 &          7 &          2 \\
        21 &        4.5 &         41 &          4 &         41 &            &         50 &        4.5 &         53 &          3 &         53 \\
        22 &          4 &         41 &          4 &         41 &            &         51 &        5.5 &          9 &          6 &          8 \\
        23 &        4.5 &         41 &          4 &         41 &            &         52 &        4.5 &         41 &          4 &         41 \\
        24 &        5.5 &         24 &          5 &         23 &            &         53 &        5.5 &         24 &          5 &         23 \\
        25 &          5 &         24 &          5 &         23 &            &         54 &          7 &          3 &          7 &          2 \\
        26 &        6.5 &         24 &          5 &         23 &            &         55 &        4.5 &         53 &          3 &         53 \\
        27 &        7.5 &          9 &          6 &          8 &            &         56 &        6.5 &          3 &          7 &          2 \\
        28 &        4.5 &         53 &          3 &         53 &            &         57 &          7 &          1 &          8 &          1 \\
        29 &          6 &          9 &          6 &          8 &            &         58 &          6 &         24 &          5 &         23 \\
\end{tabular}%
\end{table}

In Table \ref{tab:results}, the consensus rating $\xc$ is non-integral because
some of the judges assigned fractional scores (in particular they assigned
grades that are multiple of $1/2$). To appropriately handle the judges'
fractional grades, we decided to set the `grading unit' to 1/2. From an
optimization point of view, this represents no problem, since the
separation-deviation problem can be solved in any pre-specified precision.

Next we give a specific example of objects/papers whose ratings and ranking are
in conflict with several other objects/papers. In particular, paper 14 has the
highest consensus score, however this conflicts with several papers (e.g.,
paper 54) that have a lower consensus score but a higher consensus rank. The
evaluations received by papers 14 and 54 are given in Table
\ref{tab:scores14and54}. The number of papers reviewed by each judge and the
average Field Contribution (FC) they gave are given in Table
\ref{tab:judges14and54}. The adjusted FC, obtained by dividing the paper's FC
by the judge's average FC, is given in Table \ref{tab:adjusted14and54}. From
these tables we observe the following:
\begin{compenum}
\item The ordinal evaluations of paper 54 seem better than those of paper 14.
    This explains in part why paper 54 has a better consensus rank than paper
    14.
\item The average FC of paper 14 (5.6) is only slightly bigger than that of
    paper 54 (5.5). This explains in part why paper 14 has a better consensus
    score than paper 54.
\item It seems that judge 44, who evaluated paper 14, was
    remarkably lenient, while judge 30, who evaluated paper 14, was
    remarkably strict. This suggests that the FC of '5' given by
    these two judges is not comparable. Note that the adjusted FC of
    paper 14-judge 44 is of 0.71; while the adjusted FC of
    paper 54-judge 30 is of 1.39. Moreover, the average adjusted FC of
    paper 14 and 54 are 1.10 and 1.28, respectively.
\item All of this suggests that paper 54 deserves a collective
    evaluation better than that of paper 14.
\end{compenum}

In the combined aggregate rating-ranking pair, $\xdc$, which is the optimal
solution to (CAT), and in its implied ranking $\xdo$, paper 54 is rated and
ranked higher than paper 14; this, as discussed previously, seems appropriate.
In contrast, the consensus rating $\xc$ ranks paper 14 higher than 54.  This
provides some evidence that indeed the combined aggregate rating-ranking pair
better represents the judges' evaluations/opinions than the consensus rating
$\xc$, which takes into consideration only the ratings provided by the judges.

\begin{table}[H]
\begin{minipage}[t]{0.5\linewidth}
   \centering
   \vspace{3ex}
   \caption{Evaluations of papers 14 and~54.} \label{tab:scores14and54}
   \vspace{3ex}
   \small\begin{tabular}{cc|cc}%
        {\bf }&     {\bf } & {\bf Field}        & {\bf } \\
   {\bf      }&{\bf      } & {\bf Contribution} & {\bf Paper} \\
   {\bf Paper}&{\bf Judge} & {\bf    Score    } & {\bf Ranking} \\\hline
           14 &         35 &          6 &          1 \\
           14 &         23 &          6 &          1 \\
           14 &         48 &          7 &          1 \\
           14 &         57 &          4 &          4 \\
           14 &         44 &          5 &          4 \\\hline
           54 &         30 &          5 &          1 \\
           54 &         32 &          4 &          4 \\
           54 &         25 &          6 &          1 \\
           54 &         22 &          7 &          1
      \end{tabular}%
\end{minipage}\hspace{0.04\linewidth}
\begin{minipage}[t]{0.45\linewidth}
   \centering
   \caption{Evaluation statistics of the judges that evaluated papers 14 and 54.}\label{tab:judges14and54}
   \small\begin{tabular}{c|cc}%
       {\bf } & {\bf Number} & {\bf Average} \\
       {\bf } & {\bf of Papers} & {\bf Field} \\
   {\bf Judge} & {\bf Evaluated} & {\bf Contribution} \\   \hline
           35 &          4 &       4.50 \\
           23 &          4 &       4.25 \\
           48 &          4 &       5.25 \\
           57 &          4 &       5.75 \\
           44 &          5 &       7.00 \\\hline
           30 &          5 &       3.60 \\
           32 &          4 &       5.25 \\
           25 &          5 &       4.00 \\
           22 &          4 &       4.75
      \end{tabular}%
\end{minipage}
\end{table}

\begin{table}[H]
\centering
\begin{minipage}[t]{0.45\linewidth}
   \centering
   \caption{Adjusted Field Contribution received by papers 14 and 54.}\label{tab:adjusted14and54}
   \small\begin{tabular}{cc|c}%
       {\bf } &     {\bf } & {\bf Adjusted} \\
   {\bf Paper} & {\bf Judge} & {\bf Field Contribution} \\\hline
           14 &         35 &       1.33 \\
           14 &         23 &       1.41 \\
           14 &         48 &       1.33 \\
           14 &         57 &       0.70 \\
           14 &         44 &       0.71 \\\hline
           54 &         30 &       1.39 \\
           54 &         32 &       0.76 \\
           54 &         25 &       1.50 \\
           54 &         22 &       1.47 \\
   \end{tabular}%
\end{minipage}
\end{table}

Next, we use the partial order $(P,T,I)$ (created as described in Section
\ref{sec:jointAggregation}) to highlight the discrepancies between the
consensus rating $\xc$ and the consensus ranking $\xo$. Figure
\ref{fig:conflicts} gives a graphical representation of the partial order that
highlights the pairs of objects where $\xc$ and $\xo$ disagree on their
relative order (that is, those object pairs that are members of the set $I$ in
the partial order $(P,T,I)$).

From Figure \ref{fig:conflicts} we observe the following: (a) Paper 14 has the
highest consensus score, however this conflicts with several papers (e.g.,
paper 54) that have a lower consensus score but a higher consensus rank (this
agrees with the analysis given above). (b) Paper 20 (lower left corner of
Figure \ref{fig:conflicts}) should receive the lowest consensus evaluation. (c)
Although the agreement between $\xc$ and $\xo$ is not perfect, there are
subsets of papers should receive a lower (or higher) collective evaluation than
others. For example, the papers $\set{1,38,18,16,22,28,50,55}$ should receive a
collective evaluations higher than that of paper 20; lower than or equal to
that of papers $\set{5,21,23,52}$; and lower than the rest of the papers.


\begin{landscape}
\begin{figure}[H]
  \centering
  \caption{The papers (circled) are ordered
  (top to bottom) in decreasing consensus score. There is an arc
  between two papers whenever the lower rated paper has a better
  ranking than a higher rated paper.}
  \label{fig:conflicts}
  \includegraphics[width=8.62in]{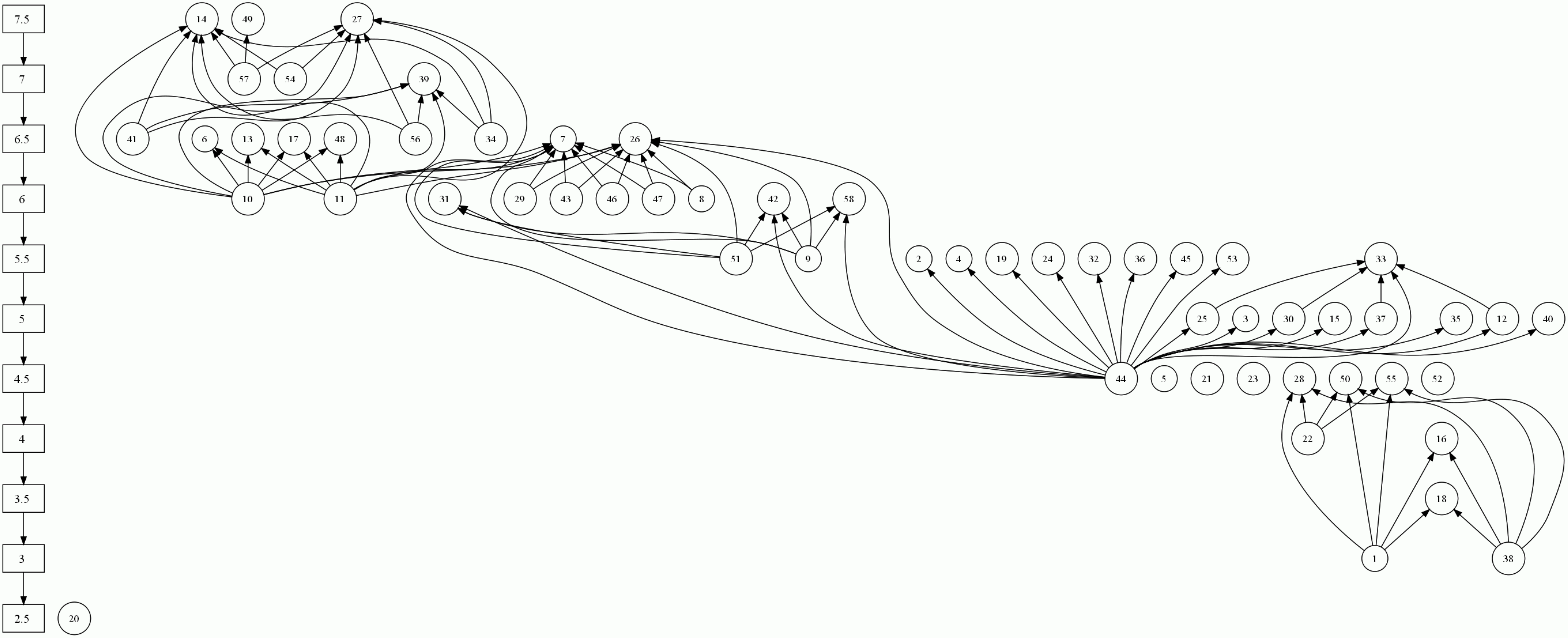}\\
\end{figure}
\end{landscape}

In the 2007 MSOM's SPC, papers 38, 14, 10, 1 and 42 had the highest
contributions to the separation penalty. As noted previously, this indicates
that these papers are those whose evaluations are not consistent/deserve
further discussion. For example, paper 38---a very low rated paper in the
consensus rating---received scores from 2 to 5 and was ranked by all but one of
the judges as their least preferred paper (see Tables \ref{tab:eval38} and
\ref{tab:judges38}). In particular, paper 38 was the second most preferred
paper of judge 9; perhaps because this judge received other papers with less
quality than paper 38? We believe this is not the case since, as shown in Table
\ref{tab:judge9}, the paper ranked last by judge 9 was paper 10. As noted
above, paper 10 is also among the highest contributors to the separation
penalty. Paper 10 received three high evaluations and 2 very low evaluations
(see Table \ref{tab:paper10}). Therefore, we believe that, in order to get a
better consensus, the scores/ranks of paper 38 and paper 10 should be discussed
by the judges assigned to these two papers.

\begin{table}[H]
\begin{minipage}[t]{0.5\linewidth}
   \centering
   \caption{Evaluations of paper 38.} \label{tab:eval38}
   \small\begin{tabular}{c|cc}%
     {\bf } & {\bf Field}        & {\bf } \\
{\bf      } & {\bf Contribution} & {\bf Paper} \\
{\bf Judge} & {\bf    Score    } & {\bf Ranking} \\\hline
         30 &          3 &          5 \\
         41 &          2 &          5 \\
         44 &          3 &          5 \\
          9 &          5 &          2 \\
         20 &          5 &          4
    \end{tabular}%
\end{minipage}
\hspace{0.04\linewidth}
\begin{minipage}[t]{0.45\linewidth}
   \centering
   \caption{Evaluation statistics of the judges that evaluated paper 38.}\label{tab:judges38}
   \small\begin{tabular}{c|cc}%
       {\bf } & {\bf Number} & {\bf Average} \\
       {\bf } & {\bf of Papers} & {\bf Field} \\
   {\bf Judge} & {\bf Evaluated} & {\bf Contribution} \\   \hline
           30 &          5 &       3.60 \\
           41 &          5 &       5.00 \\
           44 &          5 &       7.00 \\
            9 &          5 &       4.60 \\
           20 &          4 &       7.25
      \end{tabular}%
\end{minipage}
\end{table}

\begin{table}[H]
\begin{minipage}[t]{0.5\linewidth}
   \centering
   \caption{Evaluations of judge 9.}\label{tab:judge9}
   \small\begin{tabular}{cc|c}%
     {\bf } & {\bf Field}        & {\bf } \\
{\bf      } & {\bf Contribution} & {\bf Paper} \\
{\bf Paper} & {\bf    Score    } & {\bf Ranking} \\\hline
           10 &         3 &       5 \\
           19 &         4 &       3 \\
           38 &         5 &       2 \\
           50 &         4 &       3 \\
           58 &         7 &       1
   \end{tabular}%
\end{minipage}
\hspace{0.04\linewidth}
\begin{minipage}[t]{0.45\linewidth}
   \centering
   \caption{Evaluations of paper 10.}\label{tab:paper10}
   \small\begin{tabular}{cc|c}%
     {\bf } & {\bf Field}        & {\bf } \\
{\bf      } & {\bf Contribution} & {\bf Paper} \\
{\bf Judge} & {\bf    Score    } & {\bf Ranking} \\\hline
           33 &         7 &       1 \\
           41 &         7 &       1 \\
           19 &         2 &       3 \\
           15 &         6 &       1 \\
            9 &         3 &       5
   \end{tabular}%
\end{minipage}
\end{table}

Next we analyze the combined aggregate rating $\xdc$ and ranking $\xdo$
(solution to problem \eqref{eqn:simultconvex}). We make the following
observations:
\begin{compenum}
\item The consensus rating, $\xc$, has a total rating
    distance (equation \eqref{eqn:minnpck}) of 7.3611.
\item The consensus ranking, $\xo$, has a total ranking
    distance (equation \eqref{eqn:minnpks}) of 13.8500.
\item
    \begin{compenum}
    \item The combined aggregate rating, $\xdc$, has a total rating
    distance (equation \eqref{eqn:minnpck}) of 8.16667.
    \item The combined aggregate ranking, $\xdo$, has a total ranking
    distance (equation \eqref{eqn:minnpks}) of 13.9333.
    \end{compenum}
\end{compenum}
This shows that, in this case, the combined aggregate rating $\xdc$ and ranking
$\xdo$ achieve a very good compromise. In particular, $\xdc$ remains almost as
close as the consensus rating $\xc$ to the judges' ratings, and $\xdo$ remains
almost as close as the consensus ranking $\xdo$ to the judges' rankings.



\section{Concluding Remarks}\label{sec:conclusion}
We propose here a new framework for group decision making that aggregates both
cardinal and ordinal input evaluations (referred to as ratings and rankings,
respectively). Our framework consists on finding the rating-ranking pair that
minimizes the sum of the rating-distances from the rating to the given ratings
plus the sum of the ranking-distances from the ranking to the given rankings.

The effectiveness of the new framework is illustrated by ranking the
contestants of the 2007 MSOM's student paper competition. We provide evidence
that obtaining a combined aggregate cardinal and ordinal evaluations better
represents the judges' opinions as compared to a rating that aggregates only
the judges' cardinal evaluations or only the judges' ordinal evaluations.

Aggregating incomplete evaluations is challenging because the aggregate
evaluation is prone to be biased by the judges' subjective scales; for example,
objects assigned to a particularly strict (lenient) judge have an advantage
(disadvantage) compared to those objects not assigned to this specific judge.
Our framework identifies these inconsistencies in the given evaluations. This
information is helpful so that the lead decision maker can initiate an
investigation of the nature of the conflicts and act accordingly (for example,
by having the specific judges discuss, and possibly resolve, these
inconsistencies).

The problem of aggregating complete evaluations (in which all judges evaluate
all objects) is a special case of the problem of aggregating incomplete
evaluations (in which the judges are allowed to evaluate only some of the
objects). Therefore our framework is also applicable to aggregating complete
evaluations.

\section*{Acknowledgements}
The authors gratefully acknowledge J\'er\'emie Gallien, head judge of the 2007
MSOM's student paper competition, for using our methodology to aggregate the
judges evaluations in the competition.

The research of the first author is supported in part by NSF award No. CMMI-1760102.
The research of the second author is supported in part by NSF award No. OAC-1835499.

\bibliographystyle{chicago}
\bibliography{RefsSPC}

\end{document}